\documentclass[12pt]{article} 

\usepackage{amsfonts}
\usepackage{xcolor}
\usepackage{framed}

\usepackage{fullpage}
\usepackage[sort&compress,square,comma,authoryear]{natbib}

\usepackage{amsthm}
\usepackage{color}
\usepackage{amsmath}
\usepackage[sort&compress,square,comma,authoryear]{natbib}

\theoremstyle{plain}
\newtheorem{theorem}{Theorem}
\newtheorem{lemma}{Lemma}
\newtheorem{definition}{Definition}

\newtheorem{corollary}{Corollary}
\newtheorem{remark}{Remark}

\author{Yuval Dagan\thanks{Massachusetts Institute of Technology, EE\&CS, \texttt{dagan@mit.edu }}
	\and
	Gil Kur\thanks{Massachusetts Institute of Technology, EE\&CS, \texttt{gilkur@mit.edu}}
	\and
	Ohad Shamir\thanks{Weizmann Institute of Science, Israel, \texttt{ohad.shamir@weizmann.ac.il}.
		This research is supported in part by European Research Council (ERC) grant 754705.}}
\date{}

\newtheorem{note}{Note}

\newcommand{\G}{\mathrm{Gr}}
\newcommand{\dist}{\mathrm{d}}
\newcommand{\mstack}[2]{\begin{pmatrix} {#1} \\ {#2}\end{pmatrix}}

\newcommand{\sepset}{\mathcal{F}}

\newcommand\abs[1]{\left\lvert#1\right\rvert}

\newcommand\br[1]{\left(#1\right)}

\newcommand{\Sp}{\mathbb{S}}

\usepackage{color}
\newtheorem{claim}{Claim}

\newcommand{\reals}{\mathbb{R}}
\newcommand{\E}{\mathop{\mathbb{E}}}

\newcommand{\R}{\mathbb{R}}
\newcommand{\secref}[1]{Sec.~\ref{#1}}
\newcommand{\subsecref}[1]{Subsection~\ref{#1}}

\renewcommand{\eqref}[1]{Eq.~(\ref{#1})}
\newcommand{\lemref}[1]{Lemma~\ref{#1}}
\newcommand{\thmref}[1]{Thm.~\ref{#1}}

\newcommand{\appref}[1]{Appendix~\ref{#1}}

\providecommand{\Proj}{\mathrm{Proj}}

\providecommand{\spn}{\mathrm{span}}
\providecommand{\A}{\mathcal{A}}

\title{Space lower bounds for linear prediction in the streaming model}
\usepackage{times}

\begin{document}

\maketitle

\begin{abstract}%
We show that fundamental learning tasks, such as 
finding an approximate linear separator or linear regression, require memory at least 
\emph{quadratic} in the dimension, in a natural streaming setting. This 
implies that such problems cannot be solved (at least in this setting) by 
scalable memory-efficient streaming algorithms. Our 
results build 
on a memory lower bound for a simple linear-algebraic problem 
-- finding approximate null vectors -- and utilize the estimates on the packing of the Grassmannian, the manifold of all linear subspaces of fixed dimension.
\end{abstract}

\section{Introduction}

The complexity of learning, as studied in classical learning theory, is mostly 
concerned about the number of data instances required to solve a given learning 
task (a.k.a. sample complexity). However, as data becomes increasingly abundant 
and plentiful, the bottleneck in many tasks has shifted to computational 
resources, such as running time and memory usage. In particular, our 
understanding of how memory constraints affect learning performance 
is still rather limited. 

As of today, scalable supervised learning algorithms are characterized by being 
linear in the data dimension: In other 
words, the amount of required computer memory is not much larger than what is 
required to store a single data instance (represented as a vector in 
$\reals^d$). Stochastic gradient-based methods, which are based on sequentially 
processing a single or a small mini-batch of examples, are a prominent member 
of this class. In contrast, algorithms whose memory usage is super-polynomial 
in $d$ are challenging to implement for high-dimensional data. It is thus an 
important theoretical problem to understand what are the inherent limitations 
of memory-constrained algorithms. 

In this paper, we study several fundamental linear prediction problems in a 
natural streaming setting, and prove \emph{quadratic} memory lower bounds  
using any, possibly randomized algorithms (in other words, for data in 
$d$ dimensions, one needs $\Omega(d^2)$ memory in order to solve them): 
\begin{itemize}
	\item \textbf{Linear Separators:} Given a stream of $\Omega(d)$  
	unit vectors $x_1,x_2,\ldots$, which are linearly separable 
	(that is, $\min_i x_i^\top w>\gamma$ for some unit vector $w$ and margin 
	$\gamma>0$), find a linear separator. In fact, the lower bound is shown 
	even if the margin $\gamma$ is as large as $\Theta(1/\sqrt{d})$, and even if 
	the predictor is allowed to classify a small (constant) fraction of the points incorrectly.
	\item \textbf{Linear Regression:} Given a stream of $d$ labeled examples
	$\{(A_i,b_i)\}_{i=1}^{d}$ (which can be interpreted as rows of a $d\times d$ 
	matrix $A$ 
	and entries of a vector $b$), find a point $\hat{w}$ such that $\sum_i 
	(A_i^\top\hat{w}-b_i)^2$ is smaller than some universal constant.
	It also applies for algorithms that are allowed to make a pass over the stream at a random order.
	The lower bound is shown even if there exists a solution $w^*$ ($\|w^*\|\le 1$) such that $\forall 
	i,A_i w^*=b_i$ and even if $\forall i, \|A_i\| \le 1$ and $\|b\| \le 1$.
\end{itemize}
Both problems are based on a reduction from the following simple 
linear-algebraic problem:
\begin{itemize}
	\item \textbf{Approximate Null Vectors:} Given a stream of $d-1$ vectors 
	$x_1,\ldots,x_{d-1}$ in 
	$\reals^d$, sampled 
	i.i.d. from a standard Gaussian, find a unit vector approximately 
	orthogonal to 
	all of them. Specifically, we show that quadratic memory is required to 
	find a vector $\hat{w}$ such that $\frac{1}{d}\sum_i (x_i^\top\hat{w})^2$ is less than some 
	universal constant.
\end{itemize}

All of these lower bounds hold even for randomized algorithms which \emph{succeed with probability exponentially small} in $d$.
Furthermore, they are essentially tight in terms of parameter dependencies. 
First of all, in terms of memory, all of the problems are trivially solvable 
with $\widetilde{O}(d^2)$ memory (where 
$\widetilde{O}$ hides constants and logarithmic factors), simply by 
storing all the data and solving the problem offline (and in polynomial time) 
by 
phrasing them as a convex optimization problem. Moreover, our results are also 
tight in terms of the other problem parameters:
\begin{itemize}
	\item For finding an approximate linear separator on $m$ samples, this problem can be solved in $\widetilde{O}(1/\gamma^4)$ memory, by drawing a random subsample of size $\widetilde{O}(1/\gamma^2)$, storing a random projection 
	of this sample into $\widetilde{O}(1/\gamma^2)$ dimensions, finding a linear 
	separator in that 
	space, and translating it back to the original space \citep{blum2006random}.
	Thus, a memory of $\widetilde{\Theta}(1/\gamma^4)$ is sufficient, and necessary when $m = \Omega(1/\gamma^2)$, for some hard distribution over datasets.
	
	\item For the linear regression problem, we can trivially get $\sum_i 
	(A_i^\top\hat{w}-b_i)^2\leq 1$ (as opposed to some constant $\ll 1$)  by 
	picking $\hat{w}=\mathbf{0}$. 
	\item For the approximate null vector problem, it is easy to get $\frac{1}{d} \sum_i 
	(x_i^\top\hat{w})^2\approx 1$ (rather than a constant $\ll 1$) by picking $\hat{w}$ 
	uniformly at random from the unit 
	sphere.
\end{itemize}

As mentioned earlier, our results are based on the lower bound we show for the 
approximate null vector problem. We rely on the existence of a 
collection of $\exp\left(\Omega(d^2)\right)$ linear subspaces (all 
$d/2$-dimensional in $\reals^d$) which are pairwise far from each 
other, with respect to a standard distance \citep{dai2007volume}. Using 
angles between vector spaces, symmetries, and the distribution over singular 
values of random matrices, we show that any successful algorithm for the above 
tasks should not confuse between two vector spaces from that collection. To allow storing each vector space at a different memory configuration, 
approximately $\log\exp(d^2)$ memory is required.

We emphasize that our results focus on a streaming setting, where only a single pass over the examples is allowed, and refer to 
performing some task on a given set of examples. (rather than over some 
underlying distribution, in a statistical learning setting). It would be 
interesting to study whether our results can be extended to such scenarios.

\subsection*{Prior Work}\label{sec:priorwork}

As mentioned earlier, the memory complexity of learning problems has attracted 
increasing interest in recent years, and we survey some relevant results below. 
However, to the best of our knowledge, these results are different than our 
work, by either focusing on very small memory budgets (e.g. insufficient to 
store even a single example), specialized data access models (which do not, for 
instance, allow for the natural setting of examples being streamed one-by-one), 
or apply to other, fundamentally different learning problems (except the recent independent work of \citealt{sharan2019memory}, discussed below). 

In a breakthrough result, \citet{raz2016fast} proved that learning 
parities -- 
corresponding to linear regression \emph{over finite fields} -- in a 
statistical setting requires either quadratic memory or an exponential sample 
size. This was later improved and extended by several works, e.g.  
\citep{raz2017time,moshkovitz2017mixing,kol2017time,garg2017extractor,beame2017time,DBLP:conf/colt/BeameGY18, DBLP:conf/innovations/MoshkovitzM18}.
But, all 
these are specific to finite fields, rather than regression over $\reals$, where no exponential gap is known. Indeed, some of these hard problems can be solved over $\mathbb{R}$ in polynomial time and linear memory using gradient based optimization. In a recent related paper, \citet{sharan2019memory} consider the problem of performing linear regression in the statistical learning setting where a stream of examples are drawn from a distribution, and show that any algorithm that uses sub-quadratic memory exhibits a slower rate of convergence to the true solution than can be achieved without memory constraints. Their result is stronger than ours on linear regression and it was studied independently using different techniques.

Another remarkable result \citep{clarkson2009numerical} 
studies linear regression over $\reals$, but in a 
different model than ours, where individual entries of the entire dataset 
matrix arrive at an arbitrary order (rather than row-by-row), and
updates to the entries can be received (e.g. ``add $1$ to coordinate 
$(2,3)$''). 
\citet{chu1991communication} studied a model of exact computations on matrices of integer entries, where no approximation error is allowed.
Related to the problem of linear separation, but 
in a different setting than ours, \citet{guha2008tight} show that a 
streaming algorithm for finding the intersection of $n$ halfspaces in $2$ or $3$ dimensions requires $\Omega(n)$ memory. 
In \citet{Dagan2018detecting}, an $\Omega(d^2)$ memory lower bound is proven for finding correlations in $d$-dimensional distributions with optimal sample complexity. This is an unsupervised statistical learning problem quite different than the ones we study here. \citet{steinhardt2015minimax} has studied the sample complexity for memory bounded sparse linear regression.

Memory lower bounds can be reduced from communication complexity lower bounds. We list two prior works on related settings, which are incomparable to ours and cannot derive quadratic memory lower bounds in the dimension. First, \citet{kane2017communication} studied the communication complexity of classification problems, in a general setting which enables dealing with arbitrary classification problems.
Secondly, \citet{fel18learn} showed that in a distributed setting with limited communication, exponentially many samples are required to find a linear separator, if the margin is small.


We list some other linear algebraic works in streaming and communication settings over the real numbers:
\citet{DBLP:conf/soda/BalcanLW019} and \citet{zhang2015distributed} studied the problem of finding approximate matrix ranks,
\citet{DBLP:conf/icml/BravermanCKLWY18} studied Schatter $p$-norms of matrices, \citet{DBLP:conf/nips/LevinSW18} studied the problem of finding a subspace which approximates the input data, \citet{DBLP:conf/icalp/CohenNW16} studied approximate matrix product, \citet{DBLP:conf/stoc/BravermanGMNW16} studied sparse linear regression, and many other works exist. Relevant work studying related linear algebraic problems over finite fields includes \citet{li2014communication}, \citet{chu1995communication}, \citet{sun2012randomized} and many others.

%
%
%
%
%
%
%
%

\paragraph{Paper organization.} Section~\ref{sec:prel} contains preliminaries, Section~\ref{sec:mr} contains the main results, Section~\ref{sec:pr-sketch} contains the proof summary, Appendix~\ref{app:math-lem} contains auxiliary mathematical lemmas, and Appendix~\ref{sec:pr} contains the full proofs.

\section{Preliminaries} \label{sec:prel}

\paragraph{Notations.}
We use $C, C', C_1, c, c'$ etc. to denote absolute positive constants which do not depend on the dimension nor on the other problem parameters. When uppercase $C$ appears, the statement is correct for any sufficiently large constant, and when lowercase $c$ appears it holds for any sufficiently small positive value.

Here are some linear algebraic definitions: 
The unit sphere is denoted by $\mathbb{S}^{d-1} := \{ x \in \mathbb{R}^d \colon \|x\|_2=1 \}$. 
The \emph{Grassmannain}, denoted by $\G(k,d)$, is the set of all subspaces of $\mathbb{R}^d$ of dimension $k$.

We use the following standard notations:
The Euclidean norm $\| \cdot \|_2$ is denoted by $\| \cdot \|$.
Given $V \in \G(k,d)$ and $u \in \mathbb{R}^d$, $\Proj_V(u)$ denotes the projection of $u$ into $V$.

For convenience, 
given linearly independent vectors $v_1,\dots,v_{d-1} \in \mathbb{R}^d$, let $\ker(v_1 \cdots v_d)$ denote the unique unit vector orthogonal to $v_1 \cdots v_{d-1}$.

\paragraph{One-pass low memory algorithms.}
We assume a setting where samples $z_1, \dots, z_m \in Z$ are obtained one after the other in a streaming fashion, and an algorithm has to compute some function of them, the output lying in a domain $O$. There is not enough memory to store all samples: only $b$ binary bits are available. The memory configuration $s_i \in \{0,1\}^b$ after receiving $z_i$ is some function $f_i$ of the previous memory configuration $s_{i-1}$ and the sample $z_i$. Here is a formal definition:
\begin{definition}
	A \emph{one-pass algorithm} $\A$ with memory usage $b$ is a collection of functions, $(f_1, \dots, f_m, o)$, where $f_i \colon Z \times \{0,1\}^b \to \{0,1\}^b$ and $o \colon \{0,1\}^b \to O$. The output of $\A$ given the input $(z_1, \dots, z_m)$ is $o(s_m)$, where $s_m$ is defined by the recursive formula: $s_0 = 0$ and $s_i = f_i(z_i, s_{i-1})$, for $i \in \{1,\dots, m\}$.
\end{definition}
We also consider algorithms which use randomness: assume there exists a finite (but unbounded) collection of $N$ numbers drawn i.i.d uniformly from $[0,1]$ at the beginning of the execution. The algorithm is allowed to read these random numbers at any time, and they do not count towards the memory usage. Formally, these random numbers are now given to $f_i$ as additional inputs: $f_i  \colon Z \times \{0,1\}^b \times [0,1]^N \to \{0,1\}^b$.

\paragraph{Hard distributions and data arriving at a random order.}
To prove lower bounds, we show that there is some hard distribution \emph{over datasets} (over $Z^m$, rather than over $Z$), where any low memory algorithm fails. The samples $z_1, \dots, z_m$ are either assumed to be shuffled beforehand, arriving at a \emph{random order}, or at a \emph{fixed order}. Formally, we say that they arrive at a random order if for any $(z_1, \dots, z_m)$ and any permutation $\pi \colon \{1,\dots, m\} \to \{1,\dots, m\}$, the probability of $(z_1, \dots, z_m)$ to arrive equals the probability of $\left(z_{\pi(1)}, \dots, z_{\pi(m)}\right)$. While the main results on the approximate null vector problem and linear regression captures a random order of arrival, the impossibility results on linear separators requires them to arrive at a fixed order.

\paragraph{One sided communication protocols.}
This captures the setting where two parties receive inputs $z_1, z_2 \in Z$ (one input per party). The first party sends a short message based on its input. Then, the second party, upon receiving its input and looking on the message, decides on the output. We allow a finite unbounded collection of $N$ i.i.d random numbers, uniform in $[0,1]$.
\begin{definition}
	A \emph{communication protocol} $\A$ that communicates $b$ bits is a pair of functions, $f \colon Z \times [0,1]^N \to \{0,1\}^b$ and $o \colon Z \times \{0,1\}^b \times [0,1]^N \to O$. The output of $\A$ given the inputs $z_1, z_2$ and the randomness $R \in [0,1]^N$ equals $o(z_2, f(z_1,R), R)$.
\end{definition}

\paragraph{Reducing between communication protocols and one-pass algorithms.}
One can simulate a low memory algorithms using communication protocols: Fix a one-pass algorithm $\A$ with memory usage $b$, receiving samples $z_1,\dots,z_m$. Assume the corresponding communication setting, where the first party receives $z_1, \dots, z_{m/2}$ and the second party receives $z_{m/2+1},\dots,z_m$. There exists a communication protocol $\A'$ using $b$ bits of communication, which simulates $\A$, namely, given any input $(z_1,\dots,z_m)$, $\A'$ outputs the same as $\A$. Indeed, this protocol $\A'$ proceeds as follows: the first party starts simulating $\A$, feeding the samples $z_1,\dots,z_{m/2}$ into $\A$. Then, it sends the last memory configuration of $\A$, using $b$ bits. The second party continues simulating the algorithm on the points $z_{m/2+1},\dots,z_m$. Then, it outputs the same as $\A$.
Hence, any lower bound on the communication of $\A'$ derives a lower bound on the memory usage of $\A$.

\paragraph{Approximability and measurability.}
To avoid dealing with the technicalities of bit representation, we assume that the inputs are real numbers, and the algorithms are allowed to compute any \emph{measurable} function on them. However, both the upper and lower bounds apply also in the standard RAM model, where each number is rounded to logarithmically many bits. The lower bounds trivially apply, since the RAM model is weaker. The upper bounds apply as well: since we are dealing with approximate solutions and problems with large margin, rounding the numbers degrades the performance only by a negligible amount.

\paragraph{Linear separators and margin.}
Given a list of pairs $((x_i,y_i))_{i=1}^m$, where $x_i \in \mathbb{R}^d$ and $y_i \in \{-1,1\}$, we say that $w \in \mathbb{S}^{d-1}$ is a \emph{linear separator} if $w^\top x_i y_i > 0$ for all $i\in \{1,\dots,m\}$. The \emph{margin} of $w$ on this set equals $\min_{i=1}^m w^\top x_i y_i / \| x_i \|$. The \emph{margin} of the dataset is the maximal margin over $w \in \mathbb{S}^{d-1}$. A \emph{hyperplane} is any $w \in \mathbb{S}^{d-1}$ used for classification.

\section{Main Results} \label{sec:mr}

First, we discuss the approximate null vector problem, then linear separators and lastly, linear regression. 

\subsection{The approximate null vector problem (ANV)} \label{subsec:ovp}

The following result shows that any one pass algorithm which receives vectors $x_1, \dots, x_{d-1}$ and outputs a vector which is approximately orthogonal to all of them, has a memory requirement of $\Omega(d^2)$. We present two variants: one, where the vectors are drawn from a standard normal distribution, and a different variant which we is use in the reductions to linear separators and linear regression.

\begin{theorem} \label{thm:ov-mem-g}
	Let $g_1, \dots, g_{d-1}$ be i.i.d vectors drawn from $\mathcal{N}(0, I_d)$. Let $\A$ be a randomized one-pass algorithm which outputs a unit vector $\hat{w}$ such that:
	\begin{equation} \label{eq:44}
		\frac{1}{d}\sum_{i=1}^{d-1} (\hat{w}^\top g_i)^2 \le c',
	\end{equation}
	with probability at least $e^{-cd}$ (the randomness is over the algorithm and over $g_1 \cdots g_{d-1}$).
	Then, the memory usage of $\A$ is $\Omega(d^2)$.
\end{theorem}
\thmref{thm:ov-mem-g} is a direct corollary of the communication variant, \thmref{thm:gaussian}, proved in \appref{subsec:pr-ov}. A summary of the proof appears in \secref{sec:pr-sketch}.

Note that if $\hat{w}$ is drawn uniformly at random from $\mathbb{S}^{d-1}$, then $\sum_{i=1}^{d-1} (\hat{w}^\top g_i)^2 \approx d$. Hence, it is impossible to do significantly better than random, even with a tiny probability of $e^{-\Omega(d)}$.

Next, we state the second variant.
Given linearly independent vectors $v_1, \dots, v_{d-1}$. We show that it is hard to find an approximate null vector even if the first entry of $\mathrm{ker}(g_1 \cdots g_{d-1})$ (the unit vector orthogonal to $g_1 \cdots g_{d-1}$) is guaranteed to be least some constant.

\begin{theorem} \label{thm:ovp}
	Let $P$ denote the distribution over $d-1$ i.i.d uniformly drawn vectors from $\mathbb{S}^{d-1}$, $\theta_1' \cdots \theta_{d-1}'$. Let $E$ be the event that $e_1^\top \mathrm{ker}(\theta_1' \cdots \theta_{d-1}') \ge c_f$, where $c_f$ is some sufficiently small universal constant and $e_1 = (1,0,\dots,0)$.
	Assume that the input $\theta_1 \cdots \theta_{d-1}$ is drawn from $(P \mid E)$ (from the distribution $P$ conditioned on $E$). 
	Let $\A$ be a randomized one-pass algorithm which outputs a vector $\hat{w}$ that satisfies:
	\begin{equation} \label{eq:ovp-thm}
	\sum_{i=1}^{d-1} \left( \hat{w}^\top \theta_i \right)^2
	\le c_1,
	\end{equation}
	with probability at least $e^{-c_2n}$. Then, the memory usage of $\A$ is $\Omega(d^2)$.
\end{theorem}
\thmref{thm:ovp} is a direct corollary of the communication variant, \lemref{lem:cc-cf}, proved in \subsecref{subsec:pr-ov}. A summary of the proof appears in \secref{sec:pr-sketch}.

Both \thmref{thm:ov-mem-g} and \thmref{thm:ovp} follow from the following lemma, which regards the communication setting where two parties receive vector spaces from $\G(d/2,d)$ and $\G(d/2-1,d)$, respectively, and their goal is to find an approximately orthogonal $\hat{w}$.

\begin{lemma} \label{lem:G-G}
	Assume the following communication setting: the first party receives a uniformly random vector space $V$ from $\G(d/2,d)$, and the second party receives a uniformly random vector space $U$ from $\G(d/2-1, d)$. Let $\A$ be randomized one-sided communication protocol which outputs $\hat{w} \in \mathbb{S}^{d-1}$ that satisfies: 
	\[
	\max\left(
	\|\Proj_{V}(\hat{w})\|,
	\|\Proj_U(\hat{w})\|
	\right)
	\le c,
	\]
	with probability at least $e^{-c'd}$. Then, the communication contains $\Omega(d^2)$ bits.
\end{lemma}
The proof appears in \subsecref{subsec:pr-ov} and a summary appears in \secref{sec:pr-sketch}.



\subsection{Linear separators (LSP)}

Let $((x_i,y_i))_{i=1}^{2m}$ denote a dataset, where $x_i \in \mathbb{S}^{d-1}$, $y_i \in \{-1,1\}$ and $m \ge C d$, for some constant $C > 0$.
Assume that the points are separable with a margin of $\gamma = \Theta(d^{-1/2})$. Given a specific dataset, the goal of the algorithm is to find a hyperplane which classifies a large fraction of the points correctly. For the lower bounds, we will fix some hard distribution over datasets (rather than on examples, which are assumed to arrive at a fixed order). We show that any algorithm which outputs a hyperplane which classifies more than $(1-c_2)2m$ points correctly ($c_2 > 0$ is a universal constant), with non-negligible probability, requires a memory of $\Omega(d^2)$.

\begin{theorem} \label{thm:LSP-err}
	There exists a distribution over datasets $((x_i,y_i))_{i=1}^{2m}$ satisfying the above properties, such that the following holds:
	any randomized one-pass algorithm which outputs a hyperplane $\hat{w}$, that with probability at least $e^{-cd}$ classifies $(1-c_2)m$ points correctly, has a memory usage of $\Omega(d^2)$ (the randomness is over the algorithm and the distribution over datasets).
\end{theorem}

This is a direct corollary of the following communication bound, for the setting where the first party receives $((x_i,y_i))_{i=1}^m$ and the second receives the remaining $m$ examples.

\begin{theorem} \label{thm:lse}
	There exists a distribution over datasets $((x_i,y_i))_{i=1}^{2m}$ satisfying the above properties, such that the following holds:
	any randomized one-sided communication protocol $\A$ which outputs a hyperplane $\hat{w}$ that, with probability at least $e^{-c_3d}$ classifies $(1-c_2)2m$ points correctly, has a memory usage of $\Omega(d^2)$.
\end{theorem}

The proof of Theorem~\ref{thm:lse} appears in \appref{subsec:pr-lsp}, and its proof sketch appears in \secref{sec:pr-sketch}. To illustrate some proof ideas of \thmref{thm:LSP-err}, we prove a weaker version, on finding an exact separator:

\begin{theorem} \label{thm:LSP}
	There exists a distribution over datasets $((x_i,y_i))_{i=1}^{2d-2}$ satisfying the above properties, such that any one-pass algorithm $\A$ which outputs with probability at least $e^{-cd}$ a linear separator (classifying all points correctly), has a memory usage of $\Omega(d^2)$.
\end{theorem}

\begin{proof}
	We reduce \thmref{thm:LSP} from \thmref{thm:ovp}, by showing that given an algorithm $\A$ for LSP which satisfies the requirements in \thmref{thm:LSP}, one can create an algorithm $\A'$ for ANV satisfying the requirements in \thmref{thm:ovp}, with the same memory usage. 
	\thmref{thm:ovp} states that the memory usage of $\A'$ is $\Omega(d^2)$, which implies that the memory usage of $\A$ is $\Omega(d^2)$ as well and concludes the proof.
	
	Here is how $\A'$ is constructed, by simulating $\A$:
	Whenever $\mathcal{A}'$ receives a point $x_i$, it creates the points $x_{i+} = x_i + c_4 e_1 /\sqrt{d}$ and $x_{i-} = x_i - c_4 e_1 /\sqrt{d}$, where $e_1$ is the first vector in the standard basis and $c_4 = \sqrt{c_1}$ ($c_1$ is the constant defined in \eqref{eq:ovp-thm}). Then, $\mathcal{A}'$ feeds $\mathcal{A}$ with the two pairs $(x_{i+}, 1)$ and $(x_{i-}, -1)$. Once the last iteration terminates, $\mathcal{A}'$ outputs the output of $\mathcal{A}$ (assuming, without loss of generality, that $\mathcal{A}$ outputs a unit vector).
	
	Note that the algorithm $\mathcal{A}$ is assumed to operate only if the margin is $\Omega(d^{-1/2})$: our theorem is only concerned with such datasets. Luckily, $\mathcal{A}$ is fed with a sufficiently separated dataset. 
	Indeed, Theorem~\ref{thm:ovp} states that $w^* := \ker(x_1 \cdots x_{d-1})$ satisfies $w^{*\top} e_1 \ge c_f$.
	The same $w^*$ is a linear separator with margin $c_f c_4 / \sqrt{d}$:
	\[
	w^{*\top} x_{i+} 
	= w^{*\top} x_i + w^{*\top} e_1 c_4 / \sqrt{d}
	\ge c_f c_4 / \sqrt{d} ~; \qquad
	w^{*\top} x_{i-} 
	= w^{*\top} x_i - w^{*\top} e_1 c_4 / \sqrt{d}
	\le - c_f c_4 / \sqrt{d}.
	\]
	
	We are left with showing that $\mathcal{A}'$ outputs a vector with a loss of at most $c_1$, satisfying \eqref{eq:ovp-thm}.
	Indeed, since the output $\hat{w}$ of $\mathcal{A}$ is a linear separator:
	\[
	0 < \hat{w}^\top x_{i+} = \hat{w}^\top x_i + \hat{w}^\top e_1 c_4 /\sqrt{d} ~; \qquad 
	0 > \hat{w}^\top x_{i-} = \hat{w}^\top x_i - \hat{w}^\top e_1 c_4 / \sqrt{d}
	\]
	hence $|\hat{w}^\top x_i| < \hat{w}^\top e_1 c_4 / \sqrt{d} \le c_4 / \sqrt{d}$. 
	Therefore,
	$
	\sum_{i=1}^{d-1} (\hat{w}^\top x_i)^2 
	\le c_4^2
	= c_1.
	$
\end{proof}

\thmref{thm:lse} shows that when the margin is $\gamma$ and $m, d = \Theta(\gamma^{-2})$, any algorithm classifying $(1-\varepsilon)$ of the points correctly requires $\Omega(\gamma^{-4})$ memory (where $\varepsilon$ is a small constant). This bound is asymptotically tight up to logarithmic factors, and there exists a one-pass algorithm with memory $\widetilde{O}(\log^2 m/(\gamma^4\varepsilon))$ (or, $\widetilde{O}(\log^2 m/\gamma^4)$ when $\varepsilon$ is a constant). This upper bound holds for any values of $m$ and $d$, where $m$ is the sample size. It is based on the following fact: if we randomly project all points to dimension $d' = O(\log m /\gamma^2)$, with high probability the dataset will still be separable with margin $\gamma/2$ \citep{blum2006random}. We sketch this algorithm below.

First, note that if $m \gg d' (= \widetilde{\Theta}(1/\gamma^2))$, it suffices to subsample $O(d'/\varepsilon)$ points, and with high probability, any linear separator on the subsample will classify $(1-\varepsilon)$ of the points in the original dataset correctly (this follows from the sample complexity of realizable learning over $\mathbb{R}^{d'}$, see \cite{shalev2014understanding}, Sec.~6.4).


Hence, it suffices to construct an algorithm with memory $O(m \log m/\gamma^2)$ which finds a hypothesis that classifies \emph{all} points correctly. This algorithm is implemented as follows: first, a uniformly random projection $P$ from $\mathbb{R}^d$ to $\mathbb{R}^{d'}$ is drawn, where $d' = O(\log m/\gamma^2)$. The algorithm projects all points $x_i$ and stores the projection $P x_i$ up to a sufficient accuracy, together with the label $y_i$. 
Then, it finds a linear separator $w_p$ in the projected space. Lastly, it outputs a preimage of $w_p$, namely, a vector $\hat{w}$ which satisfies $P\hat{w} = w_p$. There are many preimages of $\hat{w}_p$, and we select the one which is orthogonal to the kernel of $P$. This ensures that $\hat{w}^\top x_i y_i = w_p^\top Px_i y_i > 0$, and $\hat{w}$ is a linear separator as required. Indeed, if $x_{i,k}$ is the projection of $x_i$ to the kernel of $P$ and $x_{i,p} = x_i - x_{i,k}$, the following holds:
\[
\hat{w}^\top x_i
= \hat{w}^\top x_{i,p}
= (P\hat{w})^\top (P x_{i,p})
= w_p^\top P x_i,
\]
where the second equality follows from the fact that $\hat{w}$ and $x_{i,p}$ are in the subspace orthogonal to the kernel of $P$, hence applying $P$ on them results in a rotation, and, in particular, the angle between $\hat{w}$ and $x_{i,p}$ is the same as the angle between $P \hat{w}$ and $P x_{i,p}$.

\begin{remark}
The lower bound shows that while the low-memory perceptron attains low \emph{online mistake bound}, it does not guarantee low error on the training set.
\end{remark}

\subsection{Linear regression (LR)}

Let $A$ be a real matrix of dimension $d \times d$ where each row $A_i$ satisfies $\|A_i\| \le 1$. Let $b \in \mathbb{R}^d$ where $\|b\| \le 1$. Assume that there is a solution $w^* \in \mathbb{R}^d$ with $\|w^* \| \le 1$ for the equation system $Aw = b$. We prove the following theorem, on algorithms which receive the linear equations one after the other in a random order:

\begin{theorem} \label{thm:LR}
	There exists a distribution $P$ over pairs $(A,b)$ satisfying the definition from above, where the equations arrive at a random order, such that the following holds: 
	Any randomized one-pass algorithm $\mathcal{A}$ outputting $\hat{w}$ which satisfies $\| A \hat{w} - b\|^2 \le c$ with probability at least $e^{-c' n}$, has a memory usage of $\Omega(d^2)$.
\end{theorem} 

\begin{proof}
	We reduce this theorem from \thmref{thm:ovp}, as in the proof of \thmref{thm:LSP}. Assume the existence of an algorithm $\mathcal{A}$ for LR which satisfies the conditions in \thmref{thm:LR} with $c = \min(c_1 c_f^2/4, c_f^2/4)$ and $c' = c_2$, where $c_1$, $c_f$ and $c_2$ are the constants from Theorem~\ref{thm:ovp}. 
	We will show that there exists an algorithm $\A'$ for ANV with the same memory usage, obtained by simulating $\mathcal{A}$. \thmref{thm:ovp} will imply that the memory usage of $\mathcal{A}'$ is $\Omega(d^2)$, hence the memory usage of $\mathcal{A}$ is $\Omega(d^2)$. 
	
	The algorithm $\mathcal{A}'$, given any input point $\theta_i$ for ANV ($i=1,\dots, d-1$), will feed $\mathcal{A}$ with the equation $\theta_i^\top w = 0$. Additionally, $\mathcal{A}'$ will feed $\mathcal{A}$ with the equation $e_1^\top w = c_f$, where $e_1 = (1, 0, \dots, 0)$. This equation will be fed at a uniformly random location (right after feeding $\theta_i^\top w = 0$, where $i$ is drawn uniformly at random from $\{0,1,\dots,d-1\}$). 
	After receiving the output $\hat{w}_{\mathrm{LR}}$ of $\mathcal{A}$, $\mathcal{A}'$ will normalize this vector, outputting $\hat{w} = \hat{w}_{\mathrm{LR}} / \|\hat{w}_{\mathrm{LR}}\|$.
	
	Note that the dataset $A = (\theta_1 | \cdots | \theta_i | e_1 | \theta_{i+1} | \cdots | \theta_{d-1})^\top$ and $b = (0,\dots, 0, c_f, 0, \dots, 0)^\top$ satisfies the required assumptions: each row of $A$ is of norm at most $1$ and $b$ as well. There exists a solution $w^*$ to $Aw = 0$, of $\|w^*\| \le 1$ as required: $w^* = c_f \theta_d / (e_1^\top \theta_d)$, where $\theta_d = \ker(\theta_1 \cdots \theta_{d-1})$. It is guaranteed from the requirements in \subsecref{subsec:ovp} that $e_1^\top \theta_d \ge c_f$, hence $\|w^*\| \le 1$. Also, note that the samples arrive at a random order (see definition in \secref{sec:prel}).
	
	Next, we will show that the outputted vector $\hat{w}$ is approximately orthogonal to all $\theta_i$, satisfying \eqref{eq:ovp-thm}. 
	From the guarantees of $\A$ as discussed above, it follows that with probability at least $e^{-c_2 d}$, $\| A \hat{w}_{\mathrm{LR}} - b \|^2 \le c \le \min(c_1 c_f^2/4, c_f^2/4)$. Assuming that this holds, then $\| e_1^\top \hat{w}_{\mathrm{LR}} - c_f \|^2 \le c_f^2/4$, hence $\|\hat{w}_{\mathrm{LR}}\| \ge e_1^\top \hat{w}_{\mathrm{LR}} \ge c_f / 2$. Therefore,
	\begin{equation}\label{eq:2}
		\sum_{i=1}^{d-1} (\hat{w}^\top \theta_i)^2
		= \frac{1}{\|\hat{w}_{\mathrm{LR}}\|^2}\sum_{i=1}^{d-1} (\hat{w}_{\mathrm{LR}}^\top \theta_i)^2
		\le \frac{1}{\|\hat{w}_{\mathrm{LR}}\|^2} \| A \hat{w}_{\mathrm{LR}} - b \|^2
		\le \frac{c_1 c_f^2}{4 \|\hat{w}_{\mathrm{LR}}\|^2}
		\le c_1.
	\end{equation}
	\eqref{eq:ovp-thm} is satisfied, as required,
	which concludes the reduction from LR to ANV, and the proof follows.
\end{proof}

This problem can be stated as a convex optimization over the unit ball:
\begin{equation*}
	\arg\min_{x} \| Ax - b\|^2 ~; \quad
	\text{s.t. } \| x \| \le 1.
\end{equation*}
A solution $x^*$ with zero loss is guaranteed to exist, and the choice $\hat{x} = 0$ is guaranteed to have a loss of $\|b \|^2 \le 1$. We show that in order to achieve a loss less than some constant with non-negligible probability, $\Omega(d^2)$ memory is required.
For comparison, there are several gradient-based algorithms for this problem which require memory usage of only $\widetilde{O}(d)$, but at the cost of multiple passes over the data.

\begin{remark}
	We suspect that when the condition number is small, there are efficient one-pass algorithms.
\end{remark}


\section{Proof summary} \label{sec:pr-sketch}

We sketch some of our results. The full proofs can be found in \appref{sec:pr}.

\paragraph{Proof Sketch of \lemref{lem:G-G}.}

We show that the message sent by the first party has to contain $\Omega(d^2)$ bits: There are $\exp(\Omega(d^2))$ linear subspaces in $\G(d/2,d)$ which are pairwise far from each other in a known metric over the Grassmannian \citep{dai2007volume}. The first party has to send $\log_2 \exp(\Omega(d^2)) = \Omega(d^2)$ bits to specify the vector space $V$ up to a sufficient approximation factor, otherwise the second party would not be able to find an approximately null vector. Concretely, we show the following (Lemma~\ref{lem:no-joint-sol}):
\begin{center}
	\emph{Let $V_1, V_2 \in \G(d/2,d)$ be fixed vector spaces which are far apart, and let $U$ be drawn uniformly from $\G(d/2-1,d)$. Then, with probability $1 - e^{-\Omega(d)}$, all vectors $w \in \mathbb{S}^{d-1}$ satisfy $\|\Proj_{V_1}(w)\|^2 + \|\Proj_{V_2}(w)\|^2 + \|\Proj_U(w)\|^2 = \Omega(1)$.}
\end{center}

Here is the proof outline for this statement: since $V_1$ is far from $V_2$, their orthogonal complementaries, $V_1^\perp$ and $V_2^\perp$, are far from each other. Hence, a uniformly random vector from $\mathbb{S}^{d-1} \cap V_1^\perp$ will be far from $V_2^\perp$, in expectation. Concentration of measure phenomena on the Euclidean sphere implies that we can improve from expectation, to high probability. Hence, a random vector from $\mathbb{S}^{d-1} \cap V_1^\perp$ will be far from $V_2^\perp$, with high probability.

For a typical $U$, the space of vectors $w \in \mathbb{S}^{d-1}$ satisfying $\|\Proj_{V_1}(w)\|^2 + \|\Proj_U(w)\|^2 = o(1)$ is approximately a low dimensional vector space. If $U$ is chosen uniformly at random, this vector space can be approximated by a uniformly random subspace of $V_1^\perp$ of low dimension, denoted by $W$.

A standard technique to reduce a problem from a subspace $W$ to a finite set of points is by discretization, namely, to create a $\delta$-net of $W \cap \mathbb{S}^{d-1}$ of size exponential in the dimension of $W$. When the net is defined properly and the subspace $W$ is uniformly drawn from $V_1^\perp$, each element in the $\delta$-net is drawn uniformly from the sphere as well. We apply the union bound over the net, and derive that with high probability, each member of $W$ will be far from $V_2^\perp$, i.e. the subspaces are far from each other.

To summarize: all vectors $w \in \mathbb{S}^{d-1}$ which are approximately orthogonal to $V_1$ and $U$, lie close to the subspace $W$. The subspace $W$ is far from being orthogonal to $V_2$, namely, far from $V_2^\perp$. Hence, there exists no vector which is approximately orthogonal both to $V_1$, $V_2$ and $U$.

\paragraph{Reducing Theorem~\ref{thm:ov-mem-g} from \lemref{lem:G-G}.}

We prove the communication variant of Theorem~\ref{thm:ov-mem-g} (Thm.~\ref{thm:gaussian}), where there are two parties, receiving $d/2$ and $d/2-1$ samples, respectively. We consider a scaled version, where the vectors $g_1 \cdots g_{d-1}$ are drawn $\mathcal{N}(0,I_d/d)$, and the goal is to show that a memory of $\Omega(d^2)$ is required in order to find $\hat{w}$ with $\sum_{i=1}^{d-1} (\hat{w}^\top g_i)^2 = o(1)$. We show the following (Lemma~\ref{lem:GaussianMatrix}):
\begin{center}
	\emph{Let $G$ be a matrix of dimension $d-1 \times d$ of entries $\mathcal{N}(0, I_d/d)$. Let $V$ and $U$ be the subspaces spanned by the first $d/2$ rows and the last $d/2-1$ rows of $G$, respectively. Then, with high probability, all vectors $w \in \mathbb{R}^n$ satisfy
		\[
		c \| G w \|^2
		\le \|\Proj_{V}(w)\|^2 + \|\Proj_U(w)\|^2
		\le C \| G w \|^2.
		\]
		Equivalently, if $V'$ and $U'$ are matrices with rows forming orthonormal bases for $V$ and $U$, respectively, then
		\[
		c \| G w \|^2
		\le \left\|\mstack{V'}{U'}w\right\|^2
		\le C \| G w \|^2.
		\]
	} 
\end{center}
The last statement implies that drawing orthonormal bases $V$ and $U$ is equivalent, up to absolute constants, to drawing random Gaussian vectors, and the reduction follows.

To sketch a proof of this statement, let $G_1$ and $G_2$ be the top and bottom halves of $G$, respectively. It is known that all singular values of each of these matrices are bounded by absolute constants, hence
\[
\sigma_{\min}(G_i)^2 \|\Proj_V(w)\|^2 
\le \|G_i w\|^2
\le \sigma_{\max}(G_i)^2 \|\Proj_V(w)\|^2,
\]
where $\sigma_{\min}$ and $\sigma_{\max}$ denote the minimal and maximal singular values, respectively (for $i = 1,2$).

\paragraph{Reducing Theorem~\ref{thm:ovp} from Theorem~\ref{thm:ov-mem-g}.}

We consider here the streaming variants. 
As discussed in the previous paragraph, we consider a scaled variant of Theorem~\ref{thm:ov-mem-g}, where each vector is distributed $\mathcal{N}(0,I_d/d)$. First, we claim that each such Gaussian vector is approximately of unit norm, hence we can assume they are distributed uniformly in $\mathbb{S}^{d-1}$ instead, and denote them by $\theta_1\cdots \theta_{d-1}$.

Next, \thmref{thm:ov-mem-g} states that with insufficient memory, any algorithm may succeed in outputting a vector approximately orthogonal to $\theta_1 \cdots \theta_{d-1}$ only with a tiny probability of $e^{-cd}$. Since $w^* := \ker(\theta_1 \cdots \theta_{d-1})$ is distributed uniformly in $\mathbb{S}^{d-1}$, the distribution of $e_1^\top w^*$ is known to approximately equal $\mathcal{N}(0,1/d)$ (\lemref{lem:first-coord}). In particular, $e_1^\top w^* \ge c_f$ with probability greater than $e^{-cd/2}$ (\lemref{lem:caps}). Since $e^{-cd/2} \gg e^{-cd}$, even conditioned on $e_1^\top w^* \ge  c_f$ it is impossible to find an approximate separator.

\paragraph{Reducing \thmref{thm:lse} from \lemref{lem:G-G}}

We consider a variant of \lemref{lem:G-G} where the vector $w^*$ orthogonal to $U$ and $V$ satisfies $e_1^\top w^* \ge c_f$ (\lemref{lem:otherlem}). We show that if $\A$ is a protocol for finding a linear separator, there exists a protocol $\A'$ for finding an approximate null vector with the same amount of communication.

Here is how $\A'$ is created, based on $\A$. The first party, given $V \in \G(d/2,d)$, creates an auxiliary distribution $D_V$ over pairs $(x,y)$, with the following property: Any hyperplane $w \in \mathbb{S}^{d-1}$ with low classification error on $D_V$, satisfies $\|\Proj_{V}(w)\|^2 \approx 0$. Similarly, the second party will create an auxiliary distribution $D_U$, such that any approximate separator $w$ satisfies $\|\Proj_{U}(w)\|^2 \approx 0$. In particular, any hyperplane with low error on the uniform mixture of $D_V$ and $D_U$ satisfies: $\|\Proj_{V}(w)\|^2 + \|\Proj_{U}(w)\|^2 \approx 0$.

Each party draws $m = \Omega(d)$ samples from their corresponding distribution ($D_V$ or $D_U$). Then, they simulate $\A$ to find a hyperplane $\hat{w}$ with low classification error on the mixed sample. Since the class of linear separators over $\mathbb{R}^d$ is of VC dimension $d$, $\hat{w}$ has low classification error on the mixture of $D_V$ and $D_U$, hence it satisfies $\|\Proj_{V}(w)\|^2 + \|\Proj_{U}(w)\|^2 \approx 0$, as required.
\lemref{lem:G-G} states that the communication of $\A'$ is $\Omega(d^2)$, hence the communication of $\A$ is $\Omega(d^2)$ as well.

Here is how a random pair $(x,y)$ is drawn from $D_V$ ($D_U$ is analogously defined): First a random point $x'$ is drawn uniformly from $V \cap \mathbb{S}^{d-1}$. Then, set $(x,y) = (x_+, 1)$ with probability $1/2$ and $(x,y) = (x_-,-1)$ with probability $1/2$, where $x_+ = x' + \Theta(e_1/\sqrt{d})$ and $x_- = x' - \Theta(e_1/\sqrt{d})$. 
For any fixed $w \in \mathbb{S}^{d-1}$, if $x$ is drawn uniformly from $V \cap \mathbb{S}^{d-1}$ then $w^\top x \sim \mathcal{N}(0, \|\Proj_V(w)\|^2)$ (approximately, see \lemref{lem:first-coord}). From the definition of $D_V$, any hyperplane $w$ with low classification error on $D_V$ satisfies $w^\top x \approx 0$ for most $x \in V \cap \mathbb{S}^{d-1}$, hence any such $w$ satisfies $\|\Proj_V(w)\|^2 \approx 0$, as required.

\bibliographystyle{plainnat}
\bibliography{bib}

\appendix

\section{Auxiliary Mathematical results} \label{app:math-lem}

\paragraph{Notations.}
Let $\Proj_V(v)$ denote the projection of a vector $v$ into a vector space $V$. For any subspace $V$ of $\mathbb{R}^d$ of dimension $k$, let $V^\perp$ denote the subspace of dimension $d-k$ orthogonal to $V$. For any two subspaces $U, V$ of $\mathbb{R}^d$, let $U \oplus V = \{u+v \colon u \in U, V \in V\}$ denote their direct sum.

\subsection{The Grassmannian} \label{sec:prel-gr}


There exists a unique measure over $\mathbb{S}^{d-1}$ which is uniform under rotations, namely, that satisfies: $\Pr[U A] = \Pr[A]$ for any $A \subseteq \mathbb{S}^{d-1}$ and any orthogonal (unitary) transformation $U \in O(d)$. This measure is also called the \emph{uniform} measure.

Next, we give some definitions:
\begin{definition}
	For any positive integer $d$ and $0 \le k \le d$, the set of all linear subspaces of $\mathbb{R}^d$ of dimension $k$ is denoted $\G(k,d)$, and called the \emph{Grassmannian}.
\end{definition}

\begin{definition}
	The  unique uniform probability measure (Haar measure) on the Grassmannian $\G(k,d)$ can defined as follows: Choose $ k $ vectors independently and uniformly from $ \mathbb{S}^{d-1} $ and take their linear span.
\end{definition}

Clearly, this measure is invariant under rotations, namely for any $A \subseteq \G(k,d) $ and any orthogonal transformation $ U \in O(d) $,  $\Pr(A) = \Pr(UA). $


It is known that any two lines in $\mathbb{R}^3$ have an angle between then. A generalization of this statement holds for subspaces of $\mathbb{R}^d$: For any two linear subspaces $U,V \in \G(k,d)$ we define the $k$ principal angles between them, $0 \le \theta_k \le \cdots \le \theta_1 \le \pi/2$ as follows: First, we use a fact from linear algebra that there are two orthonormal (normalized orthogonal) bases of $U$ and $V$: $v_1, \dots, v_k$ and $u_1, \dots, u_k$ respectively, such that $\langle v_i, u_j \rangle = 0$ for all $i \ne j$. Assume without loss of generality that $|\langle u_1, v_1 \rangle| \le \cdots \le |\langle u_k, v_k \rangle|$. Then, the $i$'th principal angle is $\theta_i = \arccos|\langle u_i, v_i \rangle|$.

\begin{definition}\label{prinangle}
	Let $U,V \in \G(k,d)$ be two linear subspaces and let $\theta_1, \dots, \theta_k$ denote the $k$ principal angles between them. The \emph{chordal distance} between $U$ and $V$ is defined as
	\[
	\dist(U,V) =: \sqrt{\sum_{i=1}^d \sin^2 \theta_i}.
	\]
\end{definition}

The Grassmannian can be regarded as a metric space with respect to the chordal distance. A result of \cite{dai2007volume} shows that if $k$ is a constant fraction of $d$, then there is a collection of $e^{\Omega(d^2)}$ linear subspaces in $\G(k,d)$ such that all pairwise distances are $\Omega(\sqrt{d})$. The chordal distance has also the following nice property: (see, for example, \cite{ye2016schubert})
\begin{lemma}\label{lem:comorth}
	Let $U,V \in \G(d/2,d)$ be two linear subspaces, then
	\[
	\dist(U,V) = \dist(U^{\perp},V^{\perp}).
	\]
\end{lemma}
\begin{theorem}[\cite{dai2007volume}] \label{thm:sepset}
	Let $0 < \alpha  < 1,$ then there exists a $c(\alpha) \sqrt{d}$\textbf{-separated set} $\sepset \subseteq \G(\lceil \alpha d \rceil ,d)$ of size $2^{c'(\alpha) d^2}$. Namely, for any $V \ne U\in \sepset$ it holds that $\dist(U,V) \ge c(\alpha) \sqrt{d}$.
\end{theorem}

\subsection{Random matrix theory}
Given a matrix $A$ of dimension $N \times d$, the \emph{singular values} of $A$ are the square roots of the eigenvalues of $A^\top A.$ We denote them  by $\sigma_1(A) \ge \sigma_2(A) \ge \cdots \ge \sigma_d(A) \ge 0$.
\begin{claim} \label{lem:cla-sing}
	For any matrix $A_{N \times d}$, there exists an orthonormal basis of $\mathbb{R}^d$ of singular vectors $v_1, \dots, v_d$, such that for any $\lambda_1, \dots, \lambda_d \in \mathbb{R}$,
	\[
	\left\|A \left(\sum_{i=1}^d \lambda_i v_i\right) \right\|^2
	= \sum_{i=1}^d \lambda_i^2 \sigma_i(A)^2.
	\]
\end{claim}

\begin{claim} \label{cla:restrict}
	For any matrix $A$, the collection of non-zero singular values of $A$ equals the non-zero singular values of $A^\top$. Moreover, when we restrict the matrix $A$ to operate on its rows span, then the restricted operator has the same singular values as $A^{\top}$.
\end{claim}

Let $ N \geq d $, let $ A_{N\times d} $ be a random matrix. We say that $ A $ is normal random matrix, when the all its entries are $ N(0,1) $ independent random variables. Also let $ \sigma_{\min} $ and let $ \sigma_{\max} $ be the minimal and maximal singular values of $ A$. The following are fundamental results in random matrix theory: (see for example the survey of \cite{vershynin2010introduction}) 

\begin{theorem}
	{\label{Gordon}}
	Let $ A_{d\times d} $ be a normal random matrix. The following holds for its minimal and maximal singular values:
	\[
	\sqrt{N}-\sqrt{d} \leq \E[\sigma_{\min}] \leq \E[\sigma_{\max}] \leq \sqrt{N}+\sqrt{d}.
	\]   
\end{theorem}
\begin{corollary}{\label{cor:concentarion}}
	Let $ A $ be an $ N \times d $ matrix
	whose entries are independent standard normal random variables. Then for every $  t \geq 0 $,
	with probability at least of $ 1-2e^{-\frac{t^2}{2}} $ the following holds:
	\[
	\sqrt{N}-\sqrt{d} -t \leq \sigma_{\min} \leq \sigma_{\max} \leq \sqrt{N}+\sqrt{d} + t.
	\]
\end{corollary}
The final tool that we need gives results for the mid-singular values of a normal random matrix of size $ A_{N\times d} $. The following result is from \cite{szarek1990spaces} and was generalized by \cite{wei2017upper}.

\begin{lemma}{\label{Midsingular}}
	Let $ A_{N \times d} $ be a normal random matrix and let $ 0 \leq \tau \leq 1$. Then, the following holds
	\[
	c(1-\tau)\sqrt{d} \leq \sigma_{\tau d} \leq C(1-\tau)\sqrt{d}, 
	\] 
	with probability of at least  $ 1-e^{-c\tau d} $.
\end{lemma}


\subsection{Net on the Sphere and Concentration on the sphere}


\begin{definition}[Nets, covering numbers]\label{Def:Net}
	Let ($ X, d $) be a metric space and let $ \delta > 0 $.
	A subset $ N_\delta $ of $ X $ is called a $ \delta $-net of $ X $ if for every point $ x \in X$ there exists a point $ y \in N_{\delta} $, such that $ d(x, y) \leq \delta$.
	The covering number of X at scale $ \delta $ is the size $N_\delta$ of the 
	smallest $\delta$-net of $X$. 
\end{definition}

The next lemma provides a bound on the size of a $\delta$-net of the Euclidean sphere, see for example Lemma 5.2 in \cite{vershynin2010introduction}.

\begin{lemma}{\label{lem:nets}} 
	The unit Euclidean sphere $ \mathbb{S}^{d-1} $ equipped with the Euclidean metric satisfies for every $ \delta > 0 $ that
	\[
	N_\delta \leq \left(1+ \frac 2\delta\right)^d.
	\]
\end{lemma} 
\begin{lemma}[ Lemma 5.3.5 in \cite{artstein2015asymptotic}]{\label{lem:sucaprox}}
	Let $\mathcal{N}$ be a $\delta$-net on $\Sp^{d-1},$ let $f$ be a $1$- Lipshitz function. If for any $\epsilon \in (0,1)$, we know that
	\[
	\forall x \in \mathcal{N} \quad f(x) \leq 1-\epsilon,
	\]
	then,
	\[
	\forall x \in \Sp^{d-1} \quad f(x) \leq \min\left\{\frac{1-\epsilon}{1-\delta},(1-\epsilon)+\arcsin(\delta)\right\}.
	\]
\end{lemma}
The following two Lemmas are classical results from non-asymptotic geometry, see for example \cite{artstein2015asymptotic}. The first lemma states that any Lipschitz function on $\Sp^{d-1}$ is tightly concentrated around its mean:
\begin{lemma}{\label{sphere}}
	Let $ y \sim  U(\mathbb{S}^{d-1}) $ and let $ f:\mathbb{S}^{d-1} \to \R $ be a 1-Lipschitz function. The following holds:
	\[
	\Pr(|f(y)-\E[f(y)]| \geq \epsilon) \leq 2e^{-c_1d\epsilon^2}
	\] 
\end{lemma}


The next lemmas are on the distribution of a uniformly random unit vector:
\begin{lemma} \label{lem:first-coord}
	Let $ \mathbf{\theta}^{(d)}$ be a uniformly random vector from $\mathbb{S}^{d-1}$, and let $w \in \mathbb{S}^{d-1}$. Then, as $d \to \infty$, the distribution of $\sqrt{d} w^T \theta^{(d)}$ converges in distribution to $\mathcal{N}(0,1)$, namely,
	for any $\alpha \in \mathbb{R}$,
	\[
	\lim_{d \to \infty} \Pr\left[ \sqrt{d} w^T \theta^{(d)} > \alpha \right] 
	= \Pr_{g \sim \mathcal{N}(0,1)}[g > \alpha].
	\]
	Furthermore, the convergence rate does not depend on $w$.
\end{lemma}


\begin{lemma}{\label{lem:caps}}
	Fix some constant $\alpha > 0$ and let $ \mathbf{{\theta}}$ be chosen uniformly from $\mathbb{S}^{d-1} $. Fix $w \in \mathbb{S}^{d-1}$. Then, there exists $c(\alpha) > 0$ which satisfies:
	\[
	\Pr(w^T \theta \geq c(\alpha)) \geq e^{-\alpha d}
	\]
	for any sufficiently large $n$.
\end{lemma}

\section{Proofs} \label{sec:pr}

Proof of statements related to the approximate null vector problem appear in \subsecref{subsec:pr-ov}; The proof of \thmref{thm:lse} on linear separation appears in \subsecref{subsec:pr-lsp}; and proofs of the mathematical statements appear in \subsecref{subsec:math-pr}.

\subsection{Approximate null vector problem} \label{subsec:pr-ov}

We prove results on the approximate null vector problem, providing reductions between different problem settings. Let $\mathcal{F}$ be a $\sqrt{d\delta/2}$-separated set on $\G(d/2,d)$ of size $e^{\Omega(d^2)}$, which exists from Theorem~\ref{thm:sepset}, where $\delta > 0$ is a universal constant.

\begin{lemma} \label{lem:pr-pack}
	Assume the following communication setting: the first party receives a uniformly random vector space $V$ from $\mathcal{F}$, and the second party receives a uniformly random vector space $U$ from $\G(d/2-1, d)$. Let $\A$ be a randomized one-sided communication protocol which outputs a vector $\hat{w} \in \mathbb{S}^{d-1}$ which satisfies: 
	\begin{equation} \label{eq:14}
		\max\left(
			\|\Proj_{V}(\hat{w})\|,
			\|\Proj_U(\hat{w})\|
		\right)
		< c,
	\end{equation}
	with probability at least $e^{-c'd}$. Then, the communication contains $\Omega(d^2)$ bits.
\end{lemma}

The proof of this theorem relies on the following lemma:
\begin{lemma} \label{lem:no-joint-sol}
	Let $V_1,V_2$ be \textbf{fixed} linear subspaces in $\G(d,d/2)$ with distance $\dist(V_1,V_2)^2 \ge \delta d / 2$ ($\delta > 0$ is a universal constant). Let $U$ a uniformly random subspace that is drawn from $\G(d/2-1,d).$ Then with probability of at least $ 1-e^{-cd} $, all vectors $v \in \mathbb{S}^{n-1}$ satisfy
	\begin{equation} \label{eq:15}
	\max\br{\left\| \Proj_{V_1}(v) \right\|, \left\| \Proj_{V_2}(v) \right\|, \left\| \Proj_{U}(v) \right\|} \geq c.
	\end{equation}
	for a sufficiently small universal constant $ c> 0.$
\end{lemma}


The proof appears in \secref{subsec:math-pr}.

\begin{proof}[Proof of \lemref{lem:pr-pack}]
	First, we argue that it suffices to assume that $\A$ is randomized. Indeed, if there exists a randomized algorithm which outputs an approximately null vector with probability $e^{-c'd}$, then there exists a deterministic algorithm with the same guarantee: any randomized algorithm is a distribution over deterministic algorithms, hence there has to be a fixing of the randomness which outputs an approximate null vector with probability at least $e^{-c'd}$.

	Recall that $|\mathcal{F}| \ge e^{\Omega(d^2)}$, and assume that the communication of $\A$ is at most $\log_2 |\mathcal{F}|/2$. We will show that with high probability, $\max(\Proj_V(\hat{w}), \Proj_U(\hat{w})) > c$, to conclude the proof. Denote $N = 2^b$ where $b$ is the communication $\A$, and note that $N \le \sqrt{|\mathcal{F}|}$. For each $i \in \{1,\dots, N\}$, let $\mathcal{X}_i$ denote the set of all vector spaces $V \in \mathcal{F}$ such that the first player sends the message $i$ after receiving $V$ as an input. Note that $\{\mathcal{X}_i\}_{i=1}^N$ is a partition of $\mathcal{F}$ to disjoint sets.
	
	For any $V \in \mathcal{F}$ and $U \in \G(d/2-1,d)$, let $I_{V,U}$ be the indicator of whether the protocol $\A$ on inputs $V$ and $U$ outputs $\hat{w}$ which satisfies,
	\[
		\max\left(
		\|\Proj_{V}(\hat{w})\|,
		\|\Proj_U(\hat{w})\|
		\right)
		< c,
	\]
	where $c$ is the constant from \eqref{eq:14} and \eqref{eq:15} (we define the constant $c$ in \eqref{eq:14} to equal the constant of \eqref{eq:15}).
	For any $V_1 \ne V_2 \in \mathcal{F}$ and $U \in \G(d/2-1,d)$, let $J_{V_1,V_2,U}$ be the indicator of whether \eqref{eq:15} is not satisfied, namely if there exists $v \in \mathbb{S}^{d-1}$ such that 
	\[
		\max\br{\left\| \Proj_{V_1}(v) \right\|, \left\| \Proj_{V_2}(v) \right\|, \left\| \Proj_{U}(v) \right\|} < c,
	\]
	where $c$ is the value appearing in \eqref{eq:15}.
	From Lemma~\ref{lem:no-joint-sol}, for any $V_1 \ne V_2 \in \mathcal{F}$, it holds that $\E_{U \sim \G(d/2-1,d)} J_{V_1,V_2,U} \le \xi$, where $\xi = e^{-\Omega(d)}$. Additionally, note that for all $\mathcal{X}_i$ and all $V_1,V_2 \in \mathcal{X}_i$, the output of the protocol given the pair $(V_1,U)$ equals the output given $(V_2,U)$. Hence, if $J_{V_1,V_2,U} = 0$, then either $I_{V_1,U} = 0$ or $I_{V_2,U} = 0$. In other words, $I_{V_1,U} I_{V_2,U} \le J_{V_1,V_2,U}$.
	
	Note that the probability that \eqref{eq:14} holds equals
	\begin{equation} \label{eq:16}
		\E_{V \sim \mathcal{F},~ U \sim \G(d/2-1,d)} [I_{V,U}]
		= \frac{1}{|\mathcal{F}|} \sum_{i = 1}^N \E_U\left[\sum_{V \in \mathcal{X}_i} I_{V,U}\right]
		= \frac{1}{|\mathcal{F}|} \sum_{i = 1}^N \E_U\left[K_{i,U}\right],
	\end{equation}
	where $K_{i,U} = \sum_{V \in \mathcal{X}_i} I_{V,U}$.
	For any $i \in \{1,\dots,N\}$, Jenssen's inequality implies:
	\begin{multline*}
		\frac{\left( \E_U [K_{i,U}]-1 \right)^2}{2}
		\le \frac{\E_U [K_{i,U}](\E_U [K_{i,U}]-1)}{2}
		\le \E_U\left[\frac{K_{i,U} (K_{i,U}-1)}{2}\right] \\
		= \sum_{V_1 \ne V_2 \in \mathcal{X}_i} \E_U\left[I_{V_1,U} I_{V_2,U}\right] 
		\le \sum_{V_1 \ne V_2 \in \mathcal{X}_i} \E_U\left[J_{V_1,V_2,U} \right]
		\le \binom{|\mathcal{X}_i|}{2} \xi
		\le \frac{|\mathcal{X}_i|^2 \xi}{2}.
	\end{multline*}
	Hence, $\E_U[K_{i,U}] \le 1 + \sqrt{\xi} |\mathcal{X}_i|$.
	We conclude that the right hand side of \eqref{eq:16} is bounded by
	\[
		\frac{1}{\abs{\sepset}} \sum_{i=1}^N \br{1 + \sqrt{\xi} |\mathcal{X}_i|}
		= \frac{N}{|\mathcal{F}|} + \sqrt{\xi} \le 2^{-\Omega(n)},
	\]
	using the fact that $N$ was defined to be significantly smaller than $|\mathcal{F}|$.
\end{proof}

Instead of assuming that the input of the first party arrives uniformly from $\mathcal{F}$, we can assume that it arrives uniformly from $\G(d/2,d)$, as stated in Lemma~\ref{lem:G-G}, which we prove below:

\begin{proof}[Proof of Lemma~\ref{lem:G-G}]
	We reduce from Lemma~\ref{lem:pr-pack}. Fix a protocol $\A$ which solves the setting in Lemma~\ref{lem:G-G} and we will show that there exists a protocol $\A'$ for the setting in Lemma~\ref{lem:pr-pack} with the same amount of communication. The lower on the communication of $\A'$ implies a lower bound on the communication of $\A$.
	
	Here is how $\A'$ is constructed: using the joint random bits\footnote{The parties are assumed to have shared random bits, as described in Section~\ref{sec:prel}}, the parties will draw a uniformly random rotation $R$, namely, a unitary matrix of dimension $d \times d$. Then, they simulate $\A$ as if their inputs are $R V$ and $R U$ (where $RV$ and $RU$ are the results of applying $R$ on their vector spaces). Let $w$ be the output of the simulated protocol. The second party will output $\hat{w} = R^{-1} w$.
	
	First, note that $RV$ and $RU$ are two i.i.d uniformly random vector spaces from $\G(d/2,d)$ and $\G(d/2-1,d)$, respectively, hence, the simulated protocol $\A$ receive inputs as stated in Lemma~\ref{lem:G-G}. In particular, it outputs an approximately null $w$ with a sufficiently large probability. Hence, 
	\begin{multline*}
	c \ge \max(\Proj_{RV}(w), \Proj_{RU}(w)) 
	= \max(\Proj_{RV}(R \hat{w}), \Proj_{RU}(R \hat{w})) \\
	= \max(\Proj_{V}(\hat{w}), \Proj_{U}(\hat{w})),
	\end{multline*}
	with probability probability $e^{-c'd}$, as required.
\end{proof}

Next, we prove the communication analogue of Theorem~\ref{thm:ov-mem-g}.

\begin{theorem} \label{thm:gaussian}
	Let $g_1, \dots, g_{d-1}$ be $d-1$ i.i.d vectors drawn from $\mathcal{N}(0, I_d)$.
	Assume the following communication setting: the first party receives $g_1, \dots, g_{d/2}$ and the second party receives $g_{d/2+1},\dots, g_{d-1}$. Let $\A$ be a communication protocol outputting $\hat{w}\in \mathbb{S}^{d-1}$ which satisfies:
	\[
		\sum_{i=1}^{d-1} \left( \hat{w}^\top g_i \right)^2
		\le cd,
	\]
	with probability at least $e^{-c'd}$. Then, the communication of $\A$ is $\Omega(d^2)$.
\end{theorem}

Theorem~\ref{thm:gaussian} follows from the following fact: $d/2$ random vectors are far from being linearly dependent, hence, a collection of such vectors behave as an approximate basis to a random vector space. Formally, we provide the following lemma:

\begin{lemma}\label{lem:GaussianMatrix}
	Let $g_1, g_2, \dots, g_{d-1}$ be independent random normal vectors $\mathcal{N}(0, I_d/d)$. Let $G$ be the matrix of size $(d-1) \times d$ that its  $i^{\mathrm{th}}$ row is $g_i$. Also set $V = \mathrm{span}\{ g_1, \dots, g_{d/2}\}$ and $ U = \mathrm{span}\{g_{n/2+1},\dots,g_{d-1}\}$.
	
	Then, with probability $e^{-c_2n}$, all $ v \in \mathbb{S}^{n-1}$ satisfies
	\begin{equation*}
	c_1\left\|G v \right\|^2 \le  \left\| \Proj_{V}(v) \right\|^2 + \left\| \Proj_{U}(v) \right\|^2
	\le  C_1\left\| G v \right\|^2.
	\end{equation*}
\end{lemma}

%

The proof of Lemma~\ref{lem:GaussianMatrix} appears in Subsection~\ref{subsec:math-pr}.

\begin{proof}[Proof of Theorem~\ref{thm:gaussian}]
	We will reduce to Lemma~\ref{lem:G-G}. Let $\A$ be a protocol for the setting in Lemma~\ref{lem:G-G} and we will show how to create a protocol $\A'$ for the setting in Theorem~\ref{thm:gaussian} with the same amount of communication. The lower bound on the communication of $\A'$ implies a lower bound on the communication of $\A$. 
	
	Here is how $\A'$ is created.
	Let $P_{d/2} = \mathcal{N}(0, I_d)^{d/2}$ be the distribution over $d/2$ i.i.d copies of $\mathcal{N}(0,I_d)$, and for any $V \in \G(d/2,d)$, let $E_V$ be the event that the span of these $d/2$ vectors equals $V$.
	Given an input $V \in \G(d/2,d)$, the first party will draw $g_1, \dots, g_{d/2}$ from the joint distribution $\left( P_{d/2} \mid E_V \right)$. Similarly, the second party, upon receiving $U$, will draw $g_{d/2+1},\dots,g_{d-1}$ from $\left( P_{d/2-1} \mid E_U \right)$, where $P_{d/2-1}$ and $E_U$ are similarly defined. The parties will simulate $\A$ as if the input is $g_1, \dots, g_{d-1}$, and output the vector $\hat{w}$ outputted by $\A$.
	
	For symmetrical reasons, since $U$ and $V$ are independent and uniform, the vectors $g_1, \dots, g_{d-1}$ are distributed as $d-1$ i.i.d copies from $\mathcal{N}(0,I_d)$. We assumes that $\A$ satisfies the guarantees of Thm.~\ref{thm:gaussian}, hence with probability at least $e^{-c' d}$,
	\[
		\sum_{i=1}^{d-1} \left(\hat{w}^\top g_i \right)^2 \le cd.
	\]
	With probability at least $e^{-c'd} - e^{-c_2n}$,
	\[
		\left\| \Proj_{V}(v) \right\|^2 + \left\| \Proj_{U}(v) \right\|^2
		\le C_1 \sum_{i=1}^{d-1} \left(\hat{w}^\top \frac{g_i}{\sqrt{d}} \right)^2
		\le c C_1,
	\]
	where the first inequality follows from Lemma~\ref{lem:GaussianMatrix} and holds with probability at least $1 - e^{-c_2d}$ and the second with probability at least $e^{-c'd}$. If we select the constants $c$ and $c'$ in Theorem~\ref{thm:gaussian} to be sufficiently small, we obtain that from Lemma~\ref{lem:G-G}, the memory requirement of $\A'$ is $\Omega(d^2)$, hence the memory requirement of $\A$ is $\Omega(d^2)$ as required.
\end{proof}

Lastly, we provide the communication variant of Theorem~\ref{thm:ovp}. We remind the reader that given linearly independent vectors $v_1, \dots, v_{d-1}$ we defined by $\mathrm{ker}(v_1 \cdots v_d)$ the unique unit vector orthogonal to $v_1 \cdots v_{d-1}$.

\begin{theorem} \label{lem:cc-cf}
	Let $P$ denote the distribution over $d-1$ i.i.d uniformly drawn vectors from $\mathbb{S}^{d-1}$, $\theta_1' \cdots \theta_{d-1}'$. Let $E$ be the event that $e_1^\top \mathrm{ker}(\theta_1' \cdots \theta_{d-1}') \ge c_f$, where $c_f$ is some sufficiently small universal constant.
	Let $\theta_1 \cdots \theta_{d-1}$ be random vectors drawn from $(P \mid E)$. Assume the following communication setting: the first party receives $\theta_1 \cdots \theta_{d/2}$ and the second receives $\theta_{d/2+1} \cdots \theta_{d-1}$.
	Let $\A$ be a communication protocol which outputs a vector $\hat{w}$ that satisfies:
	\[
		\sum_{i=1}^{d-1} \left( \hat{w}^\top \theta_i \right)^2
		\le c_1,
	\]
	with probability at least $e^{-c_2d}$. Then, the communication of $\A$ is $\Omega(d^2)$.
\end{theorem}

\begin{proof}
	We reduce from Thm.~\ref{thm:gaussian}: Given a protocol $\A$ for satisfying the conditions in Lemma~\ref{lem:cc-cf}, we create a protocol $\A'$ with the same amount of communication that satisfies the conditions of Thm.~\ref{thm:gaussian}. The protocol $\A'$ is defined as follows: given inputs $g_1 \cdots g_{d-1}$, the parties will normalize them to create $\theta_1 \cdots \theta_{d-1}$, where $\theta_i = g_i / \|g_i\|$. Then, they will simulate $\A$ as if their input is $\theta_1 \cdots \theta_{d-1}$. The second party will output the same output $\hat{w}$ outputted by $\A$.
	
	Assume that $\A$ satisfies the conditions in Lemma~\ref{lem:cc-cf} for sufficiently small constants $c_1$ and $c_2$. Let $c$ and $c'$ be the constants in Theorem~\ref{thm:gaussian}. 
	First, note from symmetry, that the inputs $\theta_1 \cdots \theta_{d-1}$ of $\A$ are distributed as $d-1$ i.i.d uniform copies from $\mathbb{S}^{d-1}$, hence $\mathrm{ker}(\theta_1 \cdots \theta_{d-1})$ is also uniformly distributed. From Lemma~\ref{lem:caps}, with probability at least $e^{-c'd/2}$, $e_1^\top \mathrm{ker}(\theta_1 \cdots \theta_{d-1}) \ge c_f$ (assuming $c_f$ is sufficiently small). Recall that conditioned on this holding, $\A$ is guaranteed to output an approximate separator with probability at least $e^{-c'd}$. Hence,
	\[
		\sum_{i=1}^{d-1} \left( \hat{w}^\top \theta_i \right)^2 \le c_1,
	\]
	with probability at least $e^{-c_2 d - c'd/2}$. Select $c_2$ to be sufficiently small such that $e^{-c_2 d - c'd/2} \ge 2 e^{-c'd}$ (assuming that $d$ is sufficiently large). Since each $\|g_i\|^2$ is distributed as Chi-squared with $d$ degrees of freedom, there exists a constant $C > 0$, such that with probability at least $1 - e^{-c'd}/d$, $\|g_i\|^2 \le C$. From union bound, with probability at least $1 - e^{-c'd}$, $\|g_i\|^2 \le Cd$ for all $i \in \{ 1, \dots, d-1 \}$. Hence,
	\[
		\sum_{i=1}^{d-1} \left( \hat{w}^\top g_i \right)^2
		\le Cd \sum_{i=1}^{d-1} \left( \hat{w}^\top \theta_i \right)^2
		\le Cdc_1,
	\]
	where the first inequality holds with probability at least $1 - e^{-c'd}$ and the second inequality with probability at least $2 e^{-c'd}$. Hence with probability at least $e^{-c'd}$, both inequalities hold, and if $c_1$ is sufficiently small, $\A'$ satisfies the requirements of Thm.~\ref{thm:gaussian}. In particular, the memory usage of $\A'$ is $\Omega(d^2)$.
\end{proof}

\subsection{Linear separators (Theorem~\ref{thm:lse})} \label{subsec:pr-lsp}

We prove the Theorem~\ref{thm:lse}.
First, we present an auxiliary lemma, which is a variant of the approximate null vector problem.
Given a vector space $W \in \G(d-1,d)$, denote by $\ker(W)$ the unique unit vector in $W^\perp$ and given subspaces $U$ and $V$ of $\mathbb{R}^n$, let $V\oplus U$ denote their direct sum.
\begin{lemma} \label{lem:otherlem}
	Let $P$ be a distribution over an independent pair of vector spaces: $V$ and $U$, drawn uniformly from $\G(d/2,d)$ and $\G(d/2-1,d)$, respectively. Let $E$ be the event that $e_1^\top \ker(V \oplus U) \ge c_f$, for some universal constant $c_f > 0$. Assume the communication setting where the inputs $U$ and $V$ are drawn from $(P\mid E)$.
	Let $\A$ be randomized one-sided communication protocol which outputs $\hat{w} \in \mathbb{S}^{d-1}$ that satisfies: 
	\[
	\max\left(
	\|\Proj_{V}(\hat{w})\|,
	\|\Proj_U(\hat{w})\|
	\right)
	\le c,
	\]
	with probability at least $e^{-c'd}$. Then, the communication is $\Omega(d^2)$.
\end{lemma}

Note that Lemma~\ref{lem:otherlem} is the same as Lemma~\ref{lem:G-G}, expect that the inputs are drawn from $(P\mid E)$ rather than from $P$. One can reduce Lemma~\ref{lem:otherlem} from Lemma~\ref{lem:G-G} the same way that Lemma~\ref{lem:cc-cf} follows from Theorem~\ref{thm:gaussian}.

We proceed with the following definition: Let $H$ be the set of linear separators over $\mathbb{R}^d$. Given a distribution $D$ over pairs $(x,y)$ where $x \in \mathbb{R}^d$ and $y \in \{-1,1\}$, an \emph{$\varepsilon$-approximate net} for $H$ is a finite set $S$ of pairs $(x_i,y_i)$ such that each $w \in H$ satisfies:
\[
	\left| 
	\Pr_{(x,y) \sim D}[w^\top x y > 0]
	- \Pr_{(x,y) \sim \mathrm{Uniform}(S)}[w^\top x y > 0]
	\right|
	\le \varepsilon.
\]
The following claim is equivalent to the standard uniform convergence theorems on the class of linear separators \citep{shalev2014understanding}:
\begin{claim} \label{cla:approx}
	For any $m \ge Cd/\varepsilon^2$, there exists an $\varepsilon$-approximate net $S$ of size $m$ for the hypothesis class $H$ of linear separators over $\mathbb{R}^d$ ($C>0$ is a universal constant).
\end{claim}



\begin{proof}[Proof of Theorem~\ref{thm:lse}]
	We will reduce from Lemma~\ref{lem:otherlem}. Given an algorithm $\A$ for finding a linear separator, we will create an algorithm $\A'$ for the approximate null vector problem, as follows: the first party, upon receiving $V \in \G(d/2, d)$, creates a distribution $D_V$ (as defined below), and selects a $c_2$-approximate net $S_V$ for $D_V$ of size $m \ge Cd$ (arbitrarily). Similarly, the second party, upon receiving $U \in \G(d/2-1,d)$, selects a $c_2$-approximate net $S_U$ for the corresponding distribution $D_U$. Then, they simulate the protocol $\A$ on the combined dataset $S_V \cup S_U$, and output the output $\hat{w}$ outputted by $\A$.
	
	Next, we define $D_V$. Here is how a random point $(x',y')$ is drawn from $D_V$: first, a point $x$ is drawn uniformly from $V \cap \mathbb{S}^{d-1}$. Then, with probability $1/2$, $(x',y') = (x_{+}, 1)$ and with probability $1/2$, $(x',y') = (x_-,-1)$, where $x_+ = x + c / (4\sqrt{d})$ and $x_- = x - c / (4\sqrt{d})$ (where $c$ is the constant from Lemma~\ref{lem:otherlem}). The distribution $D_U$ is defined similarly with respect to $U$.
	
	First, note that the created dataset is guaranteed to have a margin of $\Omega(\sqrt{d})$. Indeed, $\ker(V \oplus U)$ is a linear separator achieving this margin (see the proof of Theorem~\ref{thm:LSP} for a similar argument). We will show that if $\A$ finds a classifier which classifies $(1-c_2)2m$ points correctly, then $\A'$ satisfies the conditions of Lemma~\ref{lem:otherlem}, and derive the communication lower bound.
	
	We will show that any $w \in \mathbb{S}^{d-1}$ which classifies correctly a random point from $D_V$ with probability at least $1 - 3c_2$, satisfies $\|\Proj_V(w)\| < c/2$ (where $c$ is the constant from Lemma~\ref{lem:otherlem}). We will prove the contrapositive: that if $\|\Proj_V(w)\| \ge c/2$, then $w$ classifies a constant fraction of the points in $D_V$ incorrectly.
	Indeed, fix such $w$ and let $\alpha = \|\Proj_V(w)\|$. Note that if $x$ is drawn uniformly from $V$, Lemma~\ref{lem:first-coord} implies that $\sqrt{d} w^\top x$ is distributed approximately as a random variable $\mathcal{N}(0,\alpha^2)$. In particular, with constant probability, $w^\top x \ge \alpha/\sqrt{d}\ge c/(2\sqrt{d})$. For these values  of $x$, $w^\top x_- > 0$, hence, $w$ classifies $(x_-, -1)$ incorrectly. This implies that $w$ classifies incorrectly a constant fraction of the points, namely, it classifies incorrectly a random point from $D_V$ with probability $3 c_2$ of the points, if $c_2$ is sufficiently small. We conclude that any $w$ which classifies a random point from $D_V$ with probability at least $1-3c_2$, satisfies $\|\Proj_V(w) < c/2\|$.
	
	Since $S_V$ is a $c_2$ approximate net for $D_V$, any $w$ which classifies a $(1-2c_2)$ fraction of the points in $S_V$ correctly, satisfies $\|\Proj_V(w)\| < c/2$. We derive that any $w$ which classifies $(1-c_2)2m$ points correctly for the combined dataset $S_V \cup S_U$, satisfies $\|\Proj_V(w)\| < c/2$. For analogous reasoning, any such classifies satisfies $\|\Proj_U(w)\| < c/2$. Assuming that $\A$ outputs a hypothesis which classifies $(1-c_2)2m$ points correctly, this implies that $\A'$ outputs $\hat{w}$ which satisfies $\|\Proj_V(w)\| + \|\Proj_U(w)\| \le c$. From Lemma~\ref{lem:otherlem}, it follows that the communication of $\A'$ is $\Omega(d^2)$.
\end{proof}

\subsection{Proofs of the mathematical statements (\lemref{lem:GaussianMatrix} and \lemref{lem:no-joint-sol})}\label{subsec:math-pr}

\subsubsection{Proof of \lemref{lem:GaussianMatrix}}

We prove a result that is more general than  \lemref{lem:GaussianMatrix}.
\begin{lemma}{\label{genLemma}}
	Let $g_1, g_2, \dots, g_{(k_1+k_2)d}$ be independent random normal vectors $\mathcal{N}(0, I_d/d)$, where $ c_0d \le k_1d,k_2d \le d/2$ are integers. And Let $G$ be the matrix of size $(k_1+k_2)d \times d$ that its  $i^{th}$ row is $g_i$. Also set $U_1$ and $ U_2 $ be the bases of $  \mathrm{span}\{ g_1, \dots, g_{k_1}\}$ and $\mathrm{span}\{g_{k_1+1},\dots,g_{(k_1+k_2)d}\}$ respectively. 
	
	Then, for all $ t>0 $, with probability $1-2e^{-0.5\min\{k_1,k_2\}dt^2}$, all $ v \in \mathbb{S}^{d-1}$ satisfy
	\begin{multline*}
		\left(1+(1+t)\sqrt{\max\{k_1,k_2\}}\right)^{-2}\left\|G v \right\|^2 \le  \left\| \Proj_{\mathrm{span}\{U_1\}}(v) \right\|^2 + \left\| \Proj_{\mathrm{span}\{U_2\}}(v) \right\|^2 \\
		\le  \left(1-(1+t)\sqrt{\max\{k_1,k_2\}}\right)^{-2}\left\| G v \right\|^2.
	\end{multline*}
	Or equivalently, in a matrix formulation
	\begin{multline*}
		\left(1+(1+t)\sqrt{\max\{k_1,k_2\}}\right)^{-2}\left\|G v \right\|^2 \le   \left\| \mstack{U_1}{U_2}v \right\|^2 \\
		\le  \left(1-(1+t)\sqrt{\max\{k_1,k_2\}}\right)^{-2}\left\| G v \right\|^2.
	\end{multline*}
\end{lemma}
Observe that  \lemref{lem:GaussianMatrix} follows when $k_1 = 1/2$ and $k_2 = 1/2 -1/d,$ and for $t$ that is small enough.
\begin{proof}
	Let $v \in  \mathbb{S}^{d-1}.$ 
	Denote by $ W_1 = \spn\{g_1,\ldots, g_{k_1d}\} $ and $ W_2 = \spn\{g_{k_1d},\ldots, g_{(k_1+k_2)d}\}.$ Decompose $v$ in two different ways: $ v = w_1 + w_1^{\perp}$ and to $ v=w_2 + w_2^{\perp} $, where $w_i \in W_i$ and $w_i^\perp \in W_i^\perp$. Clearly, 
	\begin{equation}\label{eq:Ort}
	\begin{aligned}
	\left\| \mstack{U_1}{U_2} v \right\|^2 &= \left\| U_1 v \right\|^2 +  \left\| U_2 v \right\|^2 = \left\| U_1(w_1 + w_1^{\perp} ) \right\|^2 +  \left\| U_2 (w_2 + w_2^{\perp} ) \right\|^2 
	\\& =\left\| U_1 w_1 \right\|^2+ \left\| {U_2} w_2 \right\|^2 = \left\| \Proj_{W_1}(w_1) \right\|^2+ \left\| \Proj_{W_2}(w_2) \right\|^2,
	\end{aligned}
	\end{equation}
	where we used the fact that $ U_1,U_2 $ are orthonormal bases. Similarly, split the rows of $ G $ into two blocks with the same sizes as the number of rows of $ U_1 $ and $ U_2 $: $G = \mstack{G_1}{G_2}$. Similarly to \eqref{eq:Ort}, 
	\begin{equation}\label{eq:Guassian}
	\left\| \mstack{G_1}{G_2}v\right\|^2 = \left\| G_1 v \right\|^2+ \left\| G_2 v \right\|^2 = \left\| G_1\Proj_{W_1}(w_1) \right\|^2+ \left\| G_2\Proj_{W_2}(w_2) \right\|^2 ,
	\end{equation}
	where we use the fact that the span of the rows of $G_i$ equals to $U_i.$ Now, in order to prove the lemma, we need to connect the last two equations. Observe that $G_1,G_2$ are singular matrices, however when we restrict them to operate on the span of their rows,  the restricted linear operators have the singular values of $ G_1^{\top}$ and $G_2^{\top}$ respectively (Claim~\ref{cla:restrict}).
	Thus, by Claim~\ref{lem:cla-sing} it is enough to to bound the minimal and singular values of $G_1^{\top}$ and $G_2^{\top}$. By 
	Corollary~\ref{cor:concentarion} applied to $ G_1^{\top} $ and $ G_2^{\top} ,$  the following holds for $ t > 0 $ and $i \in \{1,2\}$: 
	\[
	1- (1+t)\sqrt{k_i} \le \sigma_{\min}(G_i^{\top}) \le \sigma_{\max}(G_i^{\top}) \leq 1 + (1+t)\sqrt{k_i},
	\]
	with probability of at least $ 1-2e^{-0.5\min\{k_1,k_2\}d t^2}. $
	Thus by Eqs. (\ref{eq:Ort}) and (\ref{eq:Guassian}) we derive that for all $v \in \mathbb{S}^{d-1}$,
	\[
	\left(1+ (1+t)\sqrt{\max\{k_1,k_2\}}\right)^{-2} \leq \frac{\left\| \mstack{U_1}{U_2} v \right\|^2}{\| Gv\|^2 } \leq \left(1 - (1+t)\sqrt{\max\{k_1,k_2\}}\right)^{-2}.
	\]
	and the the claim follows.
\end{proof} 

\subsubsection{Proof of \lemref{lem:no-joint-sol}}

\begin{note}
	We will sometimes abuse notation as follows: given some subspace $U$, the same notation will be used to denote both the subspace and an arbitrary matrix whose rows form an orthonormal basis for the same subspace.
\end{note}

We begin with a direct corollary of \lemref{lem:GaussianMatrix} and Lemma~\ref{Midsingular}.
\begin{corollary}{\label{top}}
	Fix some $U_1 \in \G(d/2,d)$ and let $U_2$ be drawn uniformly from $\G(d/2,d)$ . Fix some constant $0 < \eta < 1$, then with probability of at least $ 1-e^{c(\eta)d} $ the top $(1-\eta)d$ singular values of $\mstack{U_1}{U_2}$ are at least $c_1(\eta)$.
\end{corollary}
Recall that by Lemma \ref{lem:comorth} for any $V_1,V_2 \in \G(d/2,d)$ we know that the chordal distance satisfies that $d(V_1,V_2) = d(V_1^{\perp},V_2^{\perp}).$ We continue with another auxiliary lemma:

\begin{lemma}{\label{subspace}}
	Let $V_1^\perp$ and $V_2^\perp$ be any two vector spaces from $\G(d/2,d)$ that their distance is at least  $\sqrt{\delta d/2}$. Let $W^\perp$ be a subspace drawn uniformly from the subspaces of $V_1^\perp$ of dimension $\eta d$, where $\eta(\delta) > 0$ is a sufficiently small constant. Then, with probability at least $1 - e^{-c_3\delta^2 d}$, any $w^\perp \in W^\perp \cap \mathbb{S}^{d-1}$ satisfies that $\|\Proj_{V_2}(w^\perp)\|^2_2 \ge \delta/16$.
\end{lemma}
\begin{proof}
	Let $ U_2 $ be a matrix whose rows form an orthonormal base of $ V_2^{\perp}.$ 
	By Pythagorean law it is enough to  show that
	\[
	\|U_2w^{\perp}\|^2_2 \leq 1- \delta/16.
	\]
	
	Let $v_i^{\perp,1}, \dots, v_i^{\perp,d/2}$ be the orthonormal basis of $V_i^\perp$ for $i = 1,2$ with respect to the decomposition according to the principal angles, namely, $v_i^{\perp,j}$ corresponds to $\theta_j$ for $ i=1,2 $, see Definition \ref{prinangle}.
	Let $ y$  be a random vector chosen uniformly from $ V^{\perp}_1 \cap \mathbb{S}^{d-1}  $. Then
	\begin{align*}
		\| U_2 y \|^2
		&= \sum_{i=1}^{0.5d} \langle v_2^{\perp,i}, y \rangle^2
		=  \sum_{i=1}^{0.5d} (\langle v_2^{\perp,i},\sum_{j=1}^{0.d}\langle v_1^{\perp,j},y \rangle v_1^{\perp,j}\rangle)^2
		\\&
		=\sum_{i=1}^{0.5d} \langle y,v_1^{\perp,i} \rangle^2\langle v_2^{\perp,i},v_1^{\perp,i} \rangle^2,
	\end{align*}
	where we used the fact that $ \langle v_1^{\perp,i}v_2^{\perp,i} \rangle = 0$ for  $i\ne j.$
	
	Recall that $ y \sim  \mathrm{Unif}(V^{\perp}_1 \cap \mathbb{S}^{d-1}) \sim \mathrm{Unif}(\mathbb{S}^{d/2}) $, which implies that $ \E_y[\langle y, v \rangle^2] $ is identical for all  $ v \in V^{\perp}_1 \cap \mathbb{S}^{d-1}. $  
	Moreover,
	\[
	\E[\langle y,v_1^{\perp,i} \rangle^2] = \int_{\mathbb{S}^{d/2}}\langle y,e_1 \rangle^2 d\sigma(y) = \frac{1}{d/2}\int_{\mathbb{S}^{d/2}}\sum_{i=1}^{d/2}\langle y,e_i \rangle^2 d\sigma(y) = \frac{1}{d/2}\int_{\mathbb{S}^{d/2}}\|y\|^2 d\sigma(y) = \frac{1}{d/2}.
	\]
	Hence,
	\begin{multline*}
		\mathbb{E}[\| U_2 y \|^2] = \E[\langle y,v_1^{\perp,1} \rangle^2]\sum_{i=1}^{d/2} \langle v_2^{\perp,i},v_1^{\perp,i} \rangle^2 \\
		= \frac{1}{d/2}\sum_{i=1}^{d/2} \cos(\theta_i)^2=\frac{1}{d/2}\sum_{i=1}^{d/2}\left(1-\sin^2(\theta_i)\right) = 1 - (2/d)d(V_1^{\perp},V_2^{\perp})^2 =  1-\delta,
	\end{multline*}
	where we used the fact that the distance between the subspaces is $ \sqrt{\delta d / 2} $. Using the fact that $ \sqrt{1-\delta} \leq 1 - \frac{\delta}{2} $, we derive
	\begin{equation}\label{eq:upexp}
	\E[\| U_2 y \|]  
	\le \sqrt{\E [\| U_2 y \|^2]}
	\le \sqrt{1-\delta} \leq 1- \delta / 2.	 
	\end{equation}
	Now, since the rows of $ U_2 $ are orthonormal, its largest singular value is at most $1$, therefore, $  \|U_2x\| $ is a $1$- Lipschitz function. Hence, by Lemma \ref{sphere}, for any $ \epsilon > 0 $
	\[
	\Pr_{y \in \mathbb{S}^{d-1} \cap V_1^{\perp}}(\| U_2 y \| \leq \E[\| U_2 y \| ] + \epsilon) \geq 1 - e^{-c_1d\epsilon^2 }.
	\]
	Let $$A_{\delta} := \{ y \in \mathbb{S}^{d-1} \cap V_1^{\perp}: \| U_2 y \|  \leq 1-\delta/16 \}.$$ 
	Using \eqref{eq:upexp} and taking $ \epsilon = \min\{c_1\delta,\delta/4\} $, we derive that
	\begin{equation}\label{Eq:Adelta}
	\Pr_{y \in \mathbb{S}^{d-1} \cap V_1^{\perp}}(\| U_2 y \|  \leq 1-\delta/4 ) \geq 1- e^{-c_2d\delta^2}
	\end{equation}
	i.e. the measure of $A_\delta$ is at least $1- e^{-c_2d\delta^2}.$ Informally speaking, if we show that ``most'' of the subspaces of $V_1^{\perp} \cap \mathbb{S}^{d-1}  $ of dimension $\eta d$ lie in the set $A_{\delta},$ then we are done.  
	For this purpose, we choose a $ \delta/16 $-net $\mathcal{N}$ of  $W_0 \cap \Sp ^{d-1} $, where $W_0$ is a fixed subspace in $V_1^{\perp}$ of dimension $\eta d$ (see Definition \ref{Def:Net}). By Lemma \ref{lem:nets} we can assume that its size is bounded by 
	\[
	|\mathcal{N}| \leq \left(\frac{48}{\delta}\right)^{\eta d} \leq e^{-\ln(\delta)\eta d + \ln(32)d} \leq e^{C\ln(\delta^{-1})\eta d }.
	\]
	Now, let $W^{\perp}$ be defined as in this lemma (a uniform random subspace of $V_1^{\perp}$). It can be written as $U W_0$ for a random uniform rotation $U$ on  $V_1^{\perp}$. Note that $ \mathcal{N}_{W^{\perp}} :=  U\mathcal{N}$ is a $ \delta/16 $-net of $W^{\perp} \cap \Sp ^{d-1}$. Notice that this net is a random set of points, and moreover, each point is distributed uniformly on $\mathbb{S}^{d-1} \cap V_1^{\perp}$.
	
	Now we set $ \eta = c_2\delta^2\ln(\delta^{-1})^{-1}$ for some small enough $c_2.$ Now, in order to prove this Lemma we first estimate the probability that all the points in $\mathcal{N}_{W^{\perp}}$ lie in $A_{\delta}.$
	Using the union bound and \eqref{Eq:Adelta}, we derive that
	\[
	\Pr_{{W^{\perp}}}[\forall x \in \mathcal{N}_{W^{\perp}}: x \in A_{\delta}] \geq 1- |\mathcal{N}_{W^{\perp}}|\Pr_{y \in \mathbb{S}^{d-1} \cap V_1^{\perp}}\left[y \notin A_\delta\right] \geq 1- |\mathcal{N}_{W^{\perp}}|e^{-c_2\delta^2 d} \geq 1 - e^{-c_3\delta^2d } .
	\]
	where we used the fact that each point in the net distributed uniformly on $\mathbb{S}^{d-1} \cap V_1^{\perp}.$
	Finally, by Lemma \ref{lem:sucaprox} with $\epsilon = \delta /16$ we derive that
	\[
	\Pr_{W^{\perp}}\left[\forall {y\ \in W^{\perp}} \colon \|U_2y\| \le 1-\delta / 16\right] \geq 1 - e^{-c_2n\delta^2 }.
	\]
	Thus the claim follows for $ c(\delta) = \delta/16 $ and for $ \eta =  c_2\delta^2\ln(\delta^{-1})^{-1}.$
\end{proof}

Finally, we conclude the proof:

\begin{proof}[Proof of Lemma \ref{lem:no-joint-sol}]
	Let $ u \in \mathbb{S}^{d-1}$ and let $ \eta(\delta) > 0$ be a fixed constant such that Lemma \ref{subspace} is valid. Denote by $ W$ the subspaces of the top $(1- \eta(\delta))d$ singular vectors of $ \mstack{V_1}{U} $. Also from  Corollary \ref{top}, note that the top $(1-\eta(\delta))d$ singular values of $ \mstack{V_1}{U} $ are greater than some constant $c_5$.  Write $u$ as $u = w + w^\perp$, where $w \in W$ and $w^\perp \in W^\perp$. If $\|w\|^2 > c,$ for some constant that will be defined later, then 
	\begin{equation}\label{max1}
	\left\|\mstack{V_1}{U}u\right\| \geq c_5 \cdot \sqrt{c}, 
	\end{equation}
	and we are done.
	Otherwise, $ \|w^{\perp}\|^2 \geq 1-c.$ Then, if $\|\Proj_{V_1}(w^{\perp})\|^2 \geq (1-c)\cdot c $, the following holds:
	\begin{equation}\label{max2}
	\begin{aligned}
	\left\|\mstack{V_1}{U}u\right\|^2 &= \left\|\mstack{V_1}{U}w\right\|^2+ \left\|\mstack{V_1}{U}w^{\perp}\right\|^2 \geq \left\|\mstack{V_1}{U}w^{\perp}\right\|^2 \geq  \left\|\Proj_{V_1}(w^\perp)\right\|^2 \\&\geq (1-c)\cdot c,
	\end{aligned}
	\end{equation}
	and we are done.
	The last option is that $\|w^{\perp}\|^2 \geq 1-c$ and  $\|\Proj_{V_1^{\perp}}(w^{\perp})\|^2 \geq   (1-c)^2 $ (or equivalently $\|\Proj_{V_1}(w^{\perp})\|^2 \leq (1-c)\cdot c $). Now, we project the subspace $W^{\perp}$  on $V_1^{\perp}$, and denote the new subspace as $E$. Since the subspace  $U$ was chosen uniformly, then clearly $E$ is a uniform subspace of $V_1^{\perp}$ ($V_1$ is a fixed subspace). 
	From \lemref{subspace}, with probability $1 - e^{-c(\delta)d}$, any $e \in E \cap \mathbb{S}^{d-1}$ satisfies: $\|\Proj_{V_2} (e)\| \geq c(\delta)$, and assume for the rest of the proof that this holds. Since $\|\Proj_{V_1^{\perp}}(w^{\perp})\|^2 \geq   (1-c)^2,$ we also know that $\|\Proj_{V_2} (w^{\perp})\|^2 \geq (1-c)^2c(\delta)$. 
	Finally, 
	\begin{equation}{\label{max3}}
	\begin{aligned}
	\left\|\mstack{V_2}{U}u\right\| &= \left\|\mstack{V_2}{U}(w + w^{\perp})\right\| \geq \left\|\mstack{V_2}{U}w^{\perp}\right\| - \sqrt{c} \\&\geq \left\|\Proj_{V_2}(w^{\perp})\right\| - \sqrt{c}=  (1-c)^2c(\delta) - \sqrt{c} ,
	\end{aligned}
	\end{equation}
	By Eqs. (\ref{max1}), (\ref{max2}), (\ref{max3}) choose $c = \min\{0.01,c(\delta)^4\}$ and the claim follows.
\end{proof}


%
%
%

\end{document}